\documentclass{article}
\usepackage{arxiv}

\usepackage{amsmath,amsfonts,bm}

\def\eqref#1{equation~\ref{#1}}
\def\1{\bm{1}}

\DeclareMathAlphabet{\mathsfit}{\encodingdefault}{\sfdefault}{m}{sl}
\SetMathAlphabet{\mathsfit}{bold}{\encodingdefault}{\sfdefault}{bx}{n}

\newcommand{\E}{\mathbb{E}}

\DeclareMathOperator*{\argmax}{arg\,max}
\DeclareMathOperator*{\argmin}{arg\,min}

\usepackage[utf8]{inputenc}
\usepackage[T1]{fontenc}
\usepackage{microtype}
\usepackage{booktabs}
\usepackage{nicefrac}
\usepackage{bbm}
\usepackage{xcolor}
\usepackage{graphicx}
\usepackage{float}
\usepackage{aliascnt}
\usepackage{amsthm}
\usepackage{amssymb}
\usepackage[round]{natbib}
\usepackage{hyperref}
\usepackage{url}
\usepackage[vlined,ruled,linesnumbered]{algorithm2e}
\usepackage{cleveref}

\hypersetup{
  colorlinks=true,
  linkcolor=blue,
  citecolor=blue,
  urlcolor=blue
}

\crefname{algocf}{Algorithm}{Algorithms}
\Crefname{algocf}{Algorithm}{Algorithms}
\crefname{lemma}{lemma}{lemmas}
\Crefname{lemma}{Lemma}{Lemmas}

\newcommand{\ip}[1]{\left\langle #1 \right\rangle}
\newcommand{\bias}[1]{\mathbf{Bias}_{#1}}

\newcommand{\barS}{\bar S}
\newcommand{\Agraph}{\mathcal{A}_G}
\newcommand{\artifacturl}{\url{https://github.com/RuiyuanHuang/graphical-cross-learning-review-artifact}}
\newcommand{\artifactlink}{\href{https://github.com/RuiyuanHuang/graphical-cross-learning-review-artifact}{GitHub repository}}
\newcommand{\formalizationstatement}{The main theorem is formalized in Lean in an end-to-end development exceeding 100,000 lines of Lean code. The formalization and experiment artifacts are available in the \artifactlink: \artifacturl.}
\providecommand{\Description}[1]{}

\DeclareMathOperator{\indi}{\mathbbm{1}}
\DeclareMathOperator{\NodeIn}{N_{\mathrm{in}}}
\DeclareMathOperator{\NodeOut}{N_{\mathrm{out}}}
\DeclareMathOperator{\Prob}{\mathbb{P}}
\DeclareMathOperator{\Reg}{\mathrm{Reg}}
\DeclareMathOperator{\Time}{\mathcal{T}}

\theoremstyle{plain}
\makeatletter
\newtheorem*{rep@theorem}{\rep@title}
\newcommand{\newreptheorem}[2]{%
  \newenvironment{rep#1}[2][]{%
    \def\rep@note{##1}%
    \ifx\rep@note\@empty
      \def\rep@title{#2 \ref{##2}}%
    \else
      \def\rep@title{#2 \ref{##2} (##1)}%
    \fi
    \begin{rep@theorem}}%
  {\end{rep@theorem}}}
\makeatother

\newtheorem{theorem}{Theorem}
\newreptheorem{theorem}{Theorem}
\newtheorem{proposition}[theorem]{Proposition}
\newaliascnt{lemma}{theorem}
\newtheorem{lemma}[lemma]{Lemma}
\aliascntresetthe{lemma}

\theoremstyle{definition}

\theoremstyle{remark}

\newtheorem*{question*}{Question}

\title{Nearly Tight Bounds for Cross-Learning Contextual Bandits with Graphical Feedback}

\author{
  Ruiyuan Huang\\
  School of Data Science\\
  Fudan University\\
  Shanghai, China\\
  \texttt{RuiyuanHuang00@gmail.com}
  \And
  Zengfeng Huang\thanks{Corresponding author.}\\
  School of Data Science\\
  Fudan University\\
  Shanghai, China\\
  \texttt{huangzf@fudan.edu.cn}
}

\begin{document}
\maketitle

\begin{abstract}
Repeated first-price auctions are contextual decision problems with censored
but reusable feedback: after submitting a bid, a learner can infer the outcomes
of related bids and evaluate them under different private values. This
structure motivates cross-learning contextual bandits with graphical feedback,
where playing an arm reveals the losses of its out-neighbors in every context.
A central open question was whether, under i.i.d.\ contexts and a fixed strongly
observable feedback graph with independence number $\alpha$, one can remove
every polynomial dependence on the number of contexts while attaining the classical
graphical-bandit rate $\widetilde O(\sqrt{\alpha T})$. The question was open
even for stochastic losses and graphs in which every arm has a self-loop. We
answer it affirmatively under the stronger model of oblivious adversarial
losses and for all strongly observable graphs, including those with arms
without self-loops. The key obstruction is that a frequently played
no-self-loop arm can nevertheless have a vanishing observation probability.
Our algorithm isolates one such arm per epoch, uses a pessimistic correction to
cancel its first-order estimation drift, and shifts the losses in the FTRL
analysis to control the resulting quadratic term. It achieves expected regret
$\widetilde O(\sqrt{\alpha T})$. Controlled synthetic experiments show the
benefit of combining graph feedback with cross-learning and exhibit the
predicted scaling in both $T$ and $\alpha$.
\formalizationstatement
\end{abstract}


% 从数字广告引出一价拍卖，argue 其重要性，指出它可以被自然地建模为 graphical cross-learning contextual bandit
%
% 指出 graphical cross-learning contextual bandit 的重要性，它能被 apply 在什么样的问题上，列出一系列文章研究它
%
% 指出理论上 graphical cross-learning contextual bandit extends two standard models，thus of theoretical interest
%
% 指出 graphical cross-learning contextual bandit with stochastic contexts 是一个中心 open question，叙述它被 Han 和 Wen 列为 open question，指出它的应用价值：研究 regret 如何随着独立数 scale
%
% 指出我们解决了这个问题。指出这个问题即使在 stochastic bandit with self-loop 下都没有解决，我们是直接在更强的 adversarial bandit without self-loop 下解决。着重提到 without self-loop 的问题很困难，例如 Neu2015 解决了 with self-loop bandit 的 high probability 之后十年，才有人（Luo2023）解决了 without self-loop bandit 的 high probability。
%
% 指出我们用数值模拟验证了我们算法的有效性。并且观测到我们的算法随着 alpha 非常好地 scale。

\section{Introduction}\label{sec:introduction}

Many sequential decisions reveal more than the payoff of the chosen action.
The feedback may be censored, yet still support counterfactual evaluation across
both actions and contexts.  Repeated first-price auctions are a canonical
example: a bidder observes a private value, submits a bid, and receives only
censored information about the highest competing bid.  Nevertheless, this
feedback reveals the counterfactual outcomes of all higher bids, and each
revealed outcome can be evaluated under every possible private value.
Candidate bids therefore serve as arms in a directed feedback graph, private
values act as contexts, and the feedback cross-learns across
contexts~\citep{Han24}.

The same structure appears beyond this auction model.  Transitively closed
graphs capture one-sided observations from threshold and censoring decisions in
auctions and inventory control~\citep{MAS24}; cross-learning also models online
bidding and sleeping bandits~\citep{Bal19,Huang2025}; and similarity-based graph
feedback arises in recommendation and clinical trials~\citep{Qi24Similar}.
Across these applications, one action informs neighboring actions, and the
information remains useful under multiple contexts.  Cross-learning contextual
bandits with graphical feedback capture both forms of reuse in one model.

The model combines two familiar settings: contextual bandits and graphical bandits.  A contextual bandit observes only
the chosen arm's loss in the realized context, whereas a graphical bandit also
observes the losses of the chosen arm's out-neighbors.  Cross-learning adds a
second dimension: each observed arm is revealed in every context.  The central
question is whether an algorithm can exploit both dimensions simultaneously,
rather than paying separately for each context or ignoring the graph.

We study $K$ arms over $T$ rounds and a finite context set $[M]$.  An oblivious
adversary fixes the loss tables in advance, while contexts are drawn i.i.d.\
from an unknown distribution $\nu$.  The feedback graph $G$ is fixed and known.
It is \emph{strongly observable}: every arm either has a self-loop or is
observed by every other arm.  We write $\alpha$ for its independence number.

This model raises a central question: can cross-learning eliminate the
dependence on $M$?  Strongly observable noncontextual graphical bandits attain
$\widetilde O(\sqrt{\alpha T})$ regret~\citep{GraphAlon15}.  Treating the $M$
contexts separately instead costs $\widetilde O(\sqrt{\alpha MT})$, while
using cross-learning but ignoring the graph costs
$\widetilde O(\sqrt{KT})$.  Han et al.~\citep{Han24} initiated the combined
model and obtained $\widetilde O(\sqrt{\min\{\alpha M,K\}T})$ for stochastic
losses and i.i.d.\ contexts.  Wen et al.~\citep{MAS24} then asked whether the
context-independent rate $\widetilde O(\sqrt{\alpha T})$ is attainable.  The
question remained open even for stochastic losses and all-self-loop graphs.

We answer this question affirmatively, and in a stronger loss model than the
one in which it was posed.

\begin{theorem}[Informal]\label{thm:main}
For i.i.d. contexts, oblivious adversarial losses, and any fixed strongly
observable feedback graph $G$, our algorithm has expected regret
$\widetilde O(\sqrt{\alpha T})$, with no polynomial dependence on the number
of contexts.
\end{theorem}

Here and throughout, $\widetilde O$ hides logarithmic factors in $M$, $K$, and
$T$.  Thus ``context-independent'' means independent of $M$ up to logarithmic
factors.

The result matches the minimax rate for noncontextual strongly observable graphical
bandits up to logarithmic factors and therefore settles the stochastic-loss
question as a direct corollary.  It also removes a restriction that is often
hidden in analyses of graphical feedback: arms need not all have self-loops.
This extension is substantive.  If arm $j$ has no self-loop, then
$\NodeIn(j)=[K]\setminus\{j\}$: its loss is observed only when another arm is
played.  Assigning most of the probability to $j$ therefore makes its
observation probability arbitrarily small.  The usual graph-inverse argument
cannot control the resulting importance weights.  The same obstruction delayed
high-probability guarantees for graphical bandits: results for all-self-loop
graphs~\citep{neu2015} preceded the general strongly observable case by eight
years~\citep{LuoGraph23}.

Our algorithm isolates at most one no-self-loop arm in each epoch.  A pessimistic
correction removes its first-order estimation drift, and a translation of the
FTRL losses prevents the same arm from creating an uncontrolled quadratic term;
the remaining arms are controlled by the independence number.  Controlled
scaling experiments show a clear improvement over graph-only and context-only
baselines and are consistent with the predicted square-root dependence on both
the horizon and the graph independence number.  We also provide an end-to-end
Lean formalization exceeding 100,000 lines of code; the formalization and
experiment artifacts are available in the \artifactlink\ (\artifacturl).

\paragraph{Contributions.}
First, we resolve the context-independent regret question under oblivious
adversarial losses, a stronger model than the stochastic setting in which the
question was open.  Second, we cover every fixed strongly observable graph by
combining a special-arm selector, a pessimistic correction, and a loss shift.
Third, we develop the random-denominator and joint-stability analysis required
by the dependent estimators.  Fourth, reproducible experiments test the benefit
of both feedback structures and the predicted scaling in $T$ and $\alpha$.
Finally, we formalize our main theorem in Lean with more than 100,000 lines of code.

\subsection{Technical Overview}\label{sec:overview}

\paragraph{Known context distribution.}
We first use the known-distribution case as a warm-up.  For each context $c$,
negative-entropy FTRL maintains a distribution $p_{t,c}$,
which is mixed with uniform exploration to form $r_{t,c}$.  Because contexts are
i.i.d. and losses are chosen by an oblivious adversary, expected regret can be
averaged over the context distribution:
\[
\Reg(\pi)=\sum_{c=1}^M\nu(c)\,
\E\!\left[\sum_{t=1}^T
\ip{r_{t,c}-\pi_c,\ell_{t,c}}\right].
\]
Cross-learning makes the loss of every observed arm available to the FTRL
instance of every context.  When $\nu$ is known, the learner can therefore use
the context-independent observation probability
\[
w_t(a)=\E_{c\sim\nu}\!\left[r_{t,c}(\NodeIn(a))\right]
\]
in all contextual loss estimators.  After averaging the FTRL local-norm terms
over $c$, the numerator becomes
$\bar p_t(a)=\E_{c\sim\nu}[p_{t,c}(a)]$.  Thus the number of contexts has
already disappeared from the local-norm calculation; it remains to control the
resulting sum over arms using the feedback graph.

\paragraph{Arms without self-loops in the warm-up.}
Let
\[
\barS=\{a\in[K]:a\notin\NodeIn(a)\}
\]
be the set of arms without self-loops.  For every $a\in\barS$, strong
observability gives $\NodeIn(a)=[K]\setminus\{a\}$.  Define the population play
mass
\[
m_t(a)=\E_{c\sim\nu}[r_{t,c}(a)].
\]
The observation probability defined above then satisfies
$w_t(a)=1-m_t(a)$ for $a\in\barS$.  Hence a no-self-loop arm with large play
mass has small observation probability.  Because $m_t$ is a probability
distribution, at most one arm can have mass above $1/2$.  Following the
special-arm construction of \citet{LuoGraph23}, if $\barS\ne\emptyset$, choose
\[
j_t\in\argmax_{a\in\barS}m_t(a)
\]
using a fixed tie-breaking rule, and call $j_t$ the \emph{special arm}.  We call
the arms in $[K]\setminus\{j_t\}$ \emph{regular arms}.  If
$\barS=\emptyset$, there is no special arm and all arms are regular.  In other
words, the regular-arm set is
\[
\mathcal B_t=
\begin{cases}
[K]\setminus\{j_t\},&\barS\ne\emptyset,\\
[K],&\barS=\emptyset.
\end{cases}
\]
Up to the constant factor introduced by uniform exploration, the graph-inverse
inequality of \citet{GraphAlon15} then bounds the regular-arm sum by
\[
\sum_{a\in\mathcal B_t}\frac{\bar p_t(a)}{w_t(a)+\gamma}
=\widetilde O(\alpha),
\]
with no multiplicative dependence on $M$.

Suppose now that a special arm exists, and write
$u_t(j_t)=w_t(j_t)+\gamma$.  Implicit exploration underestimates the loss of the
special arm by a first-order term.  Feeding FTRL the pessimistic
correction $\beta/u_t(j_t)$ cancels this drift in expectation.  Directly
squaring the correction in the FTRL local norm would introduce an inverse-square
cost.  We instead subtract the full special-coordinate input from every
coordinate.  This translation leaves the negative-entropy FTRL iterate
unchanged, makes the special coordinate contribute zero to the local norm, and
weights the remaining quadratic term by the mass away from $j_t$.  Strong
observability makes this away-mass comparable to $w_t(j_t)$, removing the
inverse-square dependence.  This completes the central argument when $\nu$ is
known.

\paragraph{From known to unknown distributions.}
The known-distribution warm-up uses $\nu$ in two places: to compute the
observation probabilities $w_t(a)$ in the importance-weighted loss estimators,
and to identify the heaviest no-self-loop arm through $m_t(a)$.  When $\nu$ is
unknown, it is tempting to replace these expectations by averages over the
contexts observed so far.  This plug-in approach does not directly work.  The
current FTRL policy is itself a function of those past contexts, so the terms in
the empirical average are not fresh i.i.d. samples of the desired expectation.
Moreover, the observation probability appears in the denominator of the loss
estimator, and hence even a small estimation error can be amplified when that
probability is small.

We use the epoch construction of \citet{Sch23} to create a fixed target that can
be estimated from fresh contexts.  At the beginning of epoch $e$, a delayed
snapshot $s_{e,c}$ of the FTRL policy has already been fixed, and
$r_{e,c}$ denotes its uniformly mixed version.  The learner still plays a
distribution close to its current FTRL iterate, but uses rejection sampling to
retain feedback as if it had sampled from $r_{e,c}/2$.  Consequently, throughout
the epoch the marginal probability of observing arm $a$ is the fixed quantity
\[
w_e(a)=\frac12\E_{c\sim\nu}
\left[r_{e,c}(\NodeIn(a))\right].
\]
Unlike the time-varying quantity defined by the current policy, the integrand
above is fixed before the contexts used to estimate it are drawn.  An empirical
average of these fresh contexts is therefore an unbiased, concentrating estimate
$\widehat w_e(a)$, and a single denominator can be used throughout epoch $e$.

To preserve the required independence, consecutive rounds use the same play
distribution and are randomly assigned, after play, to two disjoint samples.
The frequency-estimation sample from the preceding epoch estimates
$w_e(a)$; the other sample supplies the retained losses used in FTRL updates and
also estimates the snapshot masses needed to select the empirically heaviest
no-self-loop arm $j_e$.  In particular, the data-dependent selector $j_e$ is
independent of the sample used to estimate its denominator.  The resulting loss
estimator is
\[
W_e(a)=\widehat w_e(a)+\frac32\gamma,
\qquad
\widehat\ell_{t,c}(a)=
\frac{2\ell_{t,c}(a)}{W_e(a)}
\indi\{B_t\in\NodeIn(a)\}.
\]
Thus snapshots and rejection sampling turn an adaptive, round-dependent
observation probability into an epochwise population quantity, while sample
splitting makes that quantity estimable without reusing the randomness that
defines it.  These ingredients handle the unknown context distribution.  The
empirical selector, together with the pessimistic correction and loss
translation from the warm-up, handles the additional graphical obstruction
created by arms without self-loops.

\paragraph{Random denominators and dependent trajectories.}
The empirical denominator $W_e$ enters both the loss estimator and the special
arm correction $\beta/W_e(j_e)$.  Consequently, the actual FTRL trajectory is
correlated with the frequency sample, and the denominator estimates are also
correlated across arms.  We introduce a population-denominator pseudo-estimator
$\widetilde\ell$ and the corresponding ideal trajectory $\widetilde p$.
Conditional on the snapshot, selector, and loss-estimation randomness, this
ideal trajectory is fixed while the frequency sample remains fresh.  Scalar
random-denominator moment bounds then apply coordinate-wise, and a joint
Softmax product-stability inequality converts the correlated trajectory error
into a sum of marginal second moments without assuming independence across
arms.  The pessimistic correction supplies the same first-order cancellation
as in the known-distribution warm-up, while the loss translation controls its
quadratic cost.  The proof follows the decomposition
\[
\Reg=\bias{1}+\bias{2}+\bias{3}+\bias{4}+\mathbf{FTRL}-\bias{5},
\]
where the first four biases respectively account for the played distribution,
population estimation, the empirical denominator on the ideal trajectory, and
the resulting trajectory perturbation.  The shifted FTRL term controls the
second-order cost, and $-\bias{5}$ supplies the pessimistic cancellation.

\subsection{Related Work}\label{sec:related-work}

\paragraph{Bandits with feedback graphs.}
Feedback graphs were introduced by \citet{MS11} as a framework interpolating
between bandit and full-information feedback.  \citet{GraphAlon15} characterized
the minimax dependence on the horizon: strongly observable graphs admit
$\widetilde O(\sqrt{\alpha T})$ regret, whereas weakly observable graphs have a
different $T^{2/3}$ regime; subsequent work sharpened the understanding of the
latter class \citep{Understanding2021}.  Gap-dependent guarantees for stochastic
graphical bandits have been developed for UCB and Thompson sampling
\citep{lykouris20a}.  High-probability adversarial guarantees progressed from
the all-self-loop setting \citep{neu2015} to general, time-varying strongly
observable graphs \citep{LuoGraph23}.  Contextual extensions have also considered
uninformed feedback graphs and bidding applications \citep{Zhang24Uninformed},
while similarity-induced feedback graphs arise in recommendation and clinical
trials \citep{Qi24Similar}.  Our graph is fixed and known; the issue here is how
its observation structure interacts with cross-learning and an unknown context
distribution.

\paragraph{Cross-learning and online auctions.}
\citet{Bal19} introduced contextual bandits with cross-learning, motivated by
repeated auctions and sleeping bandits.  This perspective has supported a
broader line of work on learning in repeated first-price auctions, including
adversarial competition and budget constraints
\citep{Han2020Adversarial,Aiauction,WangAuction}.  For adversarial losses and
i.i.d. contexts with an unknown context distribution, \citet{Sch23} obtained
the context-independent rate $\widetilde O(\sqrt{KT})$ by combining epoch
snapshots, rejection sampling, and FTRL; \citet{Huang2025} later established a
high-probability counterpart.  We use the epoch scaffold of \citet{Sch23}, but
graphical observation probabilities introduce the no-self-loop coordinate and
the dependent-denominator analysis described above.

More recent auction-learning work extends this application domain in several
complementary directions.  Non-stationary competitors motivate dynamic-regret
benchmarks \citep{Hu25Nonstationary}, while budget and return-on-investment
constraints have been studied under full and one-bit feedback
\citep{Aggarwal25NonTruthful,Li25Autobidding} and under non-stationary arrivals
\citep{WangJiang26Adaptive}.  Closest to our motivating feedback structure,
\citet{Fu26Contextual} study contextual first-price auctions with budgets when
only the winning bid is revealed and the competing bid depends on the context.
These results demonstrate the broader role of censored and structured feedback
in online bidding, but they use auction-specific stochastic or regression
models rather than the adversarial cross-learning graph model studied here.

\paragraph{Combining cross-learning and graph feedback.}
\citet{Han24} initiated the combined model through repeated first-price auctions.
They proved that adversarial contexts preclude the desired
$\widetilde O(\sqrt{\alpha T})$ rate and obtained
$\widetilde O(\sqrt{\min\{\alpha M,K\}T})$ for stochastic losses and i.i.d.
contexts.  \citet{MAS24} further characterized the adversarial-context setting
using the maximum acyclic subgraph parameter and explicitly left the
context-independent $\widetilde O(\sqrt{\alpha T})$ rate under stochastic
contexts as an open problem.  We close this gap under the stronger model of
oblivious adversarial losses.  Moreover, the result covers all fixed strongly
observable graphs, whereas the open question was unresolved even when every arm
had a self-loop and the losses were stochastic.

\section{Problem Setup and Preliminaries}\label{sec:setup}
We study a contextual $K$-armed bandit over $T$ rounds, with contexts in
$[M]$.  Before the interaction, an oblivious adversary chooses losses
$\ell_{t,c}(a)\in[0,1]$ for every round $t\in[T]$, context $c\in[M]$, and arm
$a\in[K]$.  This loss model includes stochastic losses as a special case.

Let $G=([K],E)$ be a directed feedback graph.  For each arm $a$, define its
out- and in-neighborhoods by
$\NodeOut(a)=\{v:a\to v\}$ and $\NodeIn(a)=\{v:v\to a\}$, respectively.
The independence number $\alpha=\alpha(G)$ is the size of the largest set of
vertices with no edge between any pair.  For $p\in\mathbb R^K$ and
$V\subseteq[K]$, write $p(V)=\sum_{v\in V}p(v)$.

The graph is \emph{strongly observable}: for every arm $a$, either
$a\in\NodeIn(a)$ or $\NodeIn(a)=[K]\setminus\{a\}$.  Partition the arms into
\[
S = \{a \in [K] : a \in \NodeIn(a)\}
\quad \text{and} \quad
\barS = [K] \setminus S,
\]
where $S$ contains the arms with self-loops and $\barS$ those without them.

On round $t$, a context $c_t$ is drawn i.i.d.\ from an unknown distribution
$\nu$ over $[M]$ and revealed to the learner.  The learner selects an arm
$a_t$, incurs loss $\ell_{t,c_t}(a_t)$, and observes
$\ell_{t,c}(a)$ for every context $c\in[M]$ and every
$a\in\NodeOut(a_t)$.  The latter is the cross-learning feedback.

Fix any policy $\pi:[M]\to[K]$ and identify $\pi_c$ with the corresponding
vertex of $\Delta([K])$.  Its expected regret is
\[
\Reg_T(\pi)=\E\!\left[\sum_{t=1}^T
\ell_{t,c_t}(a_t)-\ell_{t,c_t}(\pi_{c_t})\right],
\]
where the expectation is over the contexts and the learner's randomness.

\section{Algorithm}\label{sec:algorithm}

We first assume that the context distribution is known.  This warm-up exposes
the FTRL and graphical-feedback ideas without empirical denominators, and
pinpoints the two quantities that must later be estimated.

\subsection{Warm-up: known context distribution}
\label{sec:known-distribution}

Suppose temporarily that \(\nu\) is known.  For every context \(c\), let
\(p_{t,c}\) be the negative-entropy FTRL iterate and mix it with the uniform
distribution:
\[
r_{t,c}=(1-\rho)p_{t,c}+\frac{\rho}{K}\mathbf 1.
\]
Here \(\Psi(p)=\sum_{a=1}^Kp(a)\log p(a)\) denotes the negative-entropy
regularizer.
Because \(\nu\) is known, before observing \(c_t\) the learner can compute the
population action mass and observation probability
\begin{equation}
\label{eq:known-observation-probability}
m_t(a)=\E_{c\sim\nu}[r_{t,c}(a)],
\qquad
w_t(a)=\E_{c\sim\nu}[r_{t,c}(\NodeIn(a))].
\end{equation}
After drawing \(A_t\sim r_{t,c_t}\), cross-learning reveals
\(\ell_{t,c}(a)\) for every context \(c\) whenever
\(A_t\in\NodeIn(a)\).  We therefore use the implicit-exploration estimator
\[
\widehat\ell_{t,c}(a)
=\frac{\ell_{t,c}(a)}{u_t(a)}
\indi\{A_t\in\NodeIn(a)\},
\qquad
u_t(a)=w_t(a)+\gamma.
\]
Conditioned on the past,
\[
\E_t[\widehat\ell_{t,c}(a)]
=\ell_{t,c}(a)\frac{w_t(a)}{w_t(a)+\gamma}.
\]
Thus \(\gamma\) controls the estimator's variance at the cost of a small
downward bias.

There is one complication that does not arise when every arm has a self-loop.
For \(a\in\barS\), strong observability gives
\(w_t(a)=1-m_t(a)\), which can be small when \(m_t(a)\) is large.  Since at
most one no-self-loop arm can carry more than half of the population mass, set
\[
j_t\in\argmax_{a\in\barS}m_t(a)
\]
when \(\barS\ne\emptyset\), with deterministic tie-breaking.  Following the
special-arm principle of \citet{LuoGraph23}, we add the pessimistic correction
\[
b_t(a)=\frac{\beta}{u_t(j_t)}\indi\{a=j_t\};
\]
if \(\barS=\emptyset\), set \(b_t=0\).  This correction prevents the only
potentially under-observed coordinate from accumulating an uncontrolled
first-order estimation bias.  The resulting warm-up algorithm is summarized
in \Cref{alg:known}.

\begin{algorithm}
\caption{FTRL warm-up when the context distribution is known}
\label{alg:known}
\textbf{Input:} Parameters \(\eta,\gamma,\rho,\beta>0\).\\
\For{$t=1,\ldots,T$}{
    Set \(L_{t-1,c}=\sum_{s<t}(\widehat\ell_{s,c}+b_s)\) for every
    \(c\in[M]\).\\
    Set \(p_{t,c}=\argmin_{p\in\Delta([K])}
    \{\ip{p,L_{t-1,c}}+\Psi(p)/\eta\}\) for every \(c\in[M]\).\\
    Set \(r_{t,c}=(1-\rho)p_{t,c}+\rho\mathbf 1/K\) for every \(c\in[M]\).\\
    Compute \(m_t(a)\), \(w_t(a)\), and \(u_t(a)=w_t(a)+\gamma\) for every
    \(a\in[K]\) using \eqref{eq:known-observation-probability}.\\
    If \(\barS\ne\emptyset\), set
    \(j_t\in\argmax_{a\in\barS}m_t(a)\).\\
    Observe \(c_t\), draw \(A_t\sim r_{t,c_t}\), and observe
    \(\ell_{t,c}(a)\) for every \(c\in[M]\) and
    \(a\in\NodeOut(A_t)\).\\
    Set \(\widehat\ell_{t,c}(a)=
    \ell_{t,c}(a)\indi\{A_t\in\NodeIn(a)\}/u_t(a)\) for every \(c,a\).\\
    Set \(b_t(a)=\beta\indi\{a=j_t\}/u_t(j_t)\) if
    \(\barS\ne\emptyset\), and set \(b_t(a)=0\) otherwise.\\
}
\end{algorithm}

\begin{proposition}
\label{prop:known}
If \(\nu\) is known and \(G\) is strongly observable, then for appropriate
\(\eta,\gamma,\rho,\beta=\widetilde\Theta(1/\sqrt{\alpha T})\), with
\(\rho\) and \(\beta\) chosen as constant multiples of \(\gamma\),
\Cref{alg:known} satisfies
\[
\Reg(\pi)=\widetilde O(\sqrt{\alpha T})
\]
for every policy \(\pi:[M]\to[K]\).
\end{proposition}

To see the main idea, average the contextual FTRL bounds with weights
\(\nu(c)\).  The regular coordinates satisfy the usual graph-sum bound
\[
\sum_{a\ne j_t}\frac{\E_{c\sim\nu}[p_{t,c}(a)]}{u_t(a)}
=\widetilde O(\alpha).
\]
For the special coordinate, shift by the special coordinate's own FTRL input.
This does not change the FTRL iterate, but it removes that coordinate from the
quadratic term.  The remaining cost is weighted by the mass away from \(j_t\),
and strong observability gives
\(1-m_t(j_t)=w_t(j_t)\le u_t(j_t)\).  Finally, the pessimistic correction
cancels the special arm's first-order underestimation drift, up to
\(O(\beta T)\).  Together with the \(O(\rho T)\) cost of uniform mixing, this
yields
\[
\Reg(\pi)
\le \widetilde O\!\left(
\rho T+\frac{\log K}{\eta}
+(\eta+\gamma)\alpha T+\beta T
\right),
\]
which proves Proposition~\ref{prop:known} under the stated parameter choice.

The warm-up isolates two required quantities: the population observation
probabilities \(w_t\) and the masses \(m_t\) used to select a special arm.
Neither is available when \(\nu\) is unknown.  We next construct epochwise
estimates without reusing the contexts that determine the current policy.

\subsection{From known to unknown}

\paragraph{Unknown context distribution.}
Importance-weighted loss estimation requires the probability with which an arm
is observed.  If the learner used a contextual distribution $p_{t,c}$ directly,
the relevant graphical observation mass for arm $a$ would be
\[
\E_{c\sim\nu}\bigl[p_{t,c}(\NodeIn(a))\bigr].
\]
This quantity is unavailable because $\nu$ is unknown.  The plug-in estimate
\[
\frac{1}{t-1}\sum_{s<t}p_{t,c_s}(\NodeIn(a)),
\]
is invalid because $p_t$ already depends on $c_1,\ldots,c_{t-1}$ through the
FTRL updates.  The summands and the contexts being averaged are therefore
adaptively dependent, so standard i.i.d.\ concentration does not apply.

We use the epoch construction of \citet{Sch23}.  A delayed FTRL snapshot fixes
the observation rule within each epoch, while the played distribution continues
to track the current iterate.  Rejection sampling couples the two.  A random
split then assigns one round of each pair to frequency estimation and the other
to loss estimation.  The former supplies fresh contexts for the next epoch's
denominators; only the latter updates FTRL.

\paragraph{Arms without self-loops.}
We use the same special-arm treatment as in the known-distribution warm-up.
Since $\nu$ is now unknown, the special arm in each epoch is selected using an
empirical snapshot masses computed from contexts independent of those used for
the observation probabilities.  The selector and its pessimistic correction
are defined next.

\subsection{Algorithm statement}

For every context $c$, the algorithm maintains the negative-entropy FTRL
distribution
\begin{align}
\label{eq:ftrl-update}
p_{t,c}
&=\argmin_{p\in\Delta([K])}
\left\{
\ip{p,\sum_{s<t}\bigl(\widehat\ell_{s,c}+b_s\bigr)}
+\frac{1}{\eta}\Psi(p)
\right\}, \nonumber\\
&\hspace{35mm}\Psi(p)=\sum_{a=1}^K p(a)\log p(a).
\end{align}
Equivalently,
\[
p_{t,c}(a)\propto
\exp\!\left(-\eta\sum_{s<t}
\bigl(\widehat\ell_{s,c}(a)+b_s(a)\bigr)\right).
\]
Here $\eta$ is the learning rate, $\widehat\ell_{t,c}$ is the loss estimate,
and $b_t$ is the pessimistic bias.  Both are defined below.  We explicitly mix
the FTRL iterate with the uniform distribution:
\[
r^p_{t,c}=(1-\rho)p_{t,c}+\frac{\rho}{K}\mathbf 1,
\]
where $\rho\in(0,1)$ controls exploration.

\paragraph{Epoch snapshots and sample splitting.}
Assume for simplicity that $L$ is even and divides $T$, and let
$E=T/L$.  Write $\Time_e$ for the rounds in epoch $e$.  The first epoch is a
warm-up epoch, and we initialize $s_{1,c}=s_{2,c}=s_{3,c}=\mathbf 1/K$ for every
context $c$.  Thereafter, the snapshot produced in epoch $e$ is delayed by two
epochs:
\[
s_{e+2,c}=p_{eL,c}.
\]
Thus, for $e\ge 4$, the snapshot used in epoch $e$ is
$s_{e,c}=p_{(e-2)L,c}$.  Its mixed version is
\[
r_{e,c}=(1-\rho)s_{e,c}+\frac{\rho}{K}\mathbf 1.
\]

We partition each epoch into consecutive pairs.  The two rounds of a pair use
the same sampling distribution.  After both rounds have been played, a uniform
random permutation designates one as a frequency-estimation round and the other
as a loss-estimation round.  The resulting sets satisfy
\[
\Time_e^f\mathbin{\dot\cup}\Time_e^\ell=\Time_e,
\qquad
|\Time_e^f|=|\Time_e^\ell|=L/2.
\]

\paragraph{Selecting the special arm.}
Define the population and empirical snapshot masses
\[
m_e(a)=\E_{c\sim\nu}[r_{e,c}(a)],
\qquad
\widehat m_e(a)=\frac{1}{|\Time_{e-1}^\ell|}
\sum_{t\in\Time_{e-1}^\ell}r_{e,c_t}(a).
\]
If $\barS\ne\emptyset$, we choose
\[
j_e\in\argmax_{a\in\barS}\widehat m_e(a),
\]
breaking ties by a fixed deterministic rule.  The loss-estimation contexts used
for this selector are disjoint from the frequency-estimation contexts used for
the denominator below.  This independence is important because $j_e$ is a
data-dependent coordinate.  Define the set of regular arms as
$\mathcal B_e=[K]\setminus\{j_e\}$ when $\barS\ne\emptyset$, and as
$\mathcal B_e=[K]$ otherwise.  The selector need not recover the population
maximizer; it only needs to isolate the single coordinate that may have a small
observation probability.  This property is formalized in
\Cref{lem:selector}.

\paragraph{Freezing observation probabilities.}
For every $t\in\Time_e$, define the sampling distribution
\[
q_{t,c}=
\begin{cases}
r^p_{t,c},&
2r^p_{t,c}(a)\ge r_{e,c}(a)
\text{ for every }a\in[K],\\
r_{e,c},&\text{otherwise}.
\end{cases}
\]
After drawing $A_t\sim q_{t,c_t}$, the learner retains the entire graphical
feedback generated by $A_t$ with probability
\[
\frac{r_{e,c_t}(A_t)}{2q_{t,c_t}(A_t)}.
\]
The pointwise test in the definition of $q_{t,c}$ guarantees that this is a
valid probability.  Let $B_t=A_t$ if the feedback is retained and
$B_t=\bot$ otherwise.  Conditional on the context and history,
\[
\Prob\!\left(B_t\in\NodeIn(a)\mid c_t=c,\mathcal H_{t-1}\right)
=\frac12 r_{e,c}(\NodeIn(a)).
\]
Therefore, the marginal observation probability is fixed within epoch $e$ and
equals
\begin{equation}
\label{eq:epoch-observation-probability}
w_e(a)=\frac12\E_{c\sim\nu}[r_{e,c}(\NodeIn(a))].
\end{equation}
The fallback $q_{t,c}=r_{e,c}$ makes the rejection probability valid on every
sample path; the stability analysis shows that the fallback is not invoked on
the high-probability event used in the proof.

\paragraph{Frequency and loss estimation.}
Using the frequency-estimation contexts from epoch $e-1$, we estimate all
observation probabilities from the same sample:
\[
\widehat w_e(a)=\frac{1}{|\Time_{e-1}^f|}
\sum_{t\in\Time_{e-1}^f}
\frac12 r_{e,c_t}(\NodeIn(a)),
\qquad a\in[K].
\]
This is an unbiased estimator of $w_e(a)$.  We use the stabilized denominator
\[
W_e(a)=\widehat w_e(a)+\frac32\gamma,
\]
where $\gamma>0$ is an implicit exploration parameter.  For a
loss-estimation round $t\in\Time_e^\ell$, cross-learning makes the retained
loss of every observed arm available for every context, so we set
\[
\widehat\ell_{t,c}(a)
=\frac{2\ell_{t,c}(a)}{W_e(a)}
\indi\{B_t\in\NodeIn(a)\},
\qquad c\in[M],\ a\in[K].
\]
The factor (2) compensates for selecting only one loss-estimation round from
each pair.  We set $\widehat\ell_{t,c}=0$ on frequency-estimation rounds.  Finally, if a
special arm exists, we define
\[
b_t(a)=\frac{\beta}{W_e(j_e)}
\indi\{a=j_e,\ t\in\Time_e^\ell\};
\]
otherwise $b_t=0$.  FTRL is updated with the shifted input
$\widehat\ell_{t,c}+b_t$.  The additive $3\gamma/2$ in $W_e$ is the standard
implicit-exploration device used to control random importance denominators
\citep{neu2015}.

\begin{algorithm}
\caption{FTRL with unknown context distribution and graphical feedback}
\label{alg:unknown}
\textbf{Input:} Parameters $\eta,\gamma,\rho,\beta>0$ and an even epoch length
$L$ dividing $T$.\\
Set $s_{1,c}=s_{2,c}=s_{3,c}=\mathbf 1/K$ for every $c\in[M]$.\\
Set $\widehat w_2(a)=\widehat m_2(a)=0$ for every $a\in[K]$, and compute
$r_{2,c}=(1-\rho)s_{2,c}+\rho\mathbf 1/K$ for every $c$.\\
\tcp{Warm-up epoch: construct the quantities used in epoch 2.}
\For{$t=1,3,\ldots,L-1$}{
    \For{$t'\in\{t,t+1\}$}{
        Observe $c_{t'}$ and play $A_{t'}\sim s_{1,c_{t'}}$.\\
    }
    $(t_f,t_\ell)\leftarrow\mathsf{RandPerm}(t,t+1)$.\\
    \For{$a\in[K]$}{
        $\widehat w_2(a)\leftarrow\widehat w_2(a)
        +r_{2,c_{t_f}}(\NodeIn(a))/L$.\\
        $\widehat m_2(a)\leftarrow\widehat m_2(a)
        +2r_{2,c_{t_\ell}}(a)/L$.\\
    }
}
Set $W_2(a)=\widehat w_2(a)+3\gamma/2$ for every $a$.\\
If $\barS\ne\emptyset$, set
$j_2\in\argmax_{a\in\barS}\widehat m_2(a)$.\\
\For{$e=2,\ldots,E$}{
    Set $\widehat w_{e+1}(a)=\widehat m_{e+1}(a)=0$ for every $a\in[K]$.\\
    Compute $r_{e+1,c}=(1-\rho)s_{e+1,c}+\rho\mathbf 1/K$ for every $c$.\\
    \For{$t=(e-1)L+1,(e-1)L+3,\ldots,eL-1$}{
        Compute $p_{t,c}$ by \eqref{eq:ftrl-update} and
        $r^p_{t,c}=(1-\rho)p_{t,c}+\rho\mathbf 1/K$ for every $c$.\\
        Set $q_{t,c}=r^p_{t,c}$ if
        $2r^p_{t,c}(a)\ge r_{e,c}(a)$ for every $a$; otherwise set
        $q_{t,c}=r_{e,c}$.\\
        \For{$t'\in\{t,t+1\}$}{
            Observe $c_{t'}$ and play $A_{t'}\sim q_{t,c_{t'}}$.\\
            Observe $\ell_{t',c}(a)$ for every $c\in[M]$ and
            $a\in\NodeOut(A_{t'})$.\\
        }
        $(t_f,t_\ell)\leftarrow\mathsf{RandPerm}(t,t+1)$.\\
        Draw $S_t\sim\operatorname{Bernoulli}
        \!\left(r_{e,c_{t_\ell}}(A_{t_\ell})/
        (2q_{t,c_{t_\ell}}(A_{t_\ell}))\right)$.\\
        Set $B_{t_\ell}=A_{t_\ell}$ if $S_t=1$, and set
        $B_{t_\ell}=\bot$ otherwise.\\
        Set $\widehat\ell_{t_f,c}=0$ and $b_{t_f}=0$ for every $c$.\\
        \For{$a\in[K]$}{
            $\widehat w_{e+1}(a)\leftarrow\widehat w_{e+1}(a)
            +r_{e+1,c_{t_f}}(\NodeIn(a))/L$.\\
            $\widehat m_{e+1}(a)\leftarrow\widehat m_{e+1}(a)
            +2r_{e+1,c_{t_\ell}}(a)/L$.\\
            $\widehat\ell_{t_\ell,c}(a)\leftarrow
            2\ell_{t_\ell,c}(a)\indi\{B_{t_\ell}\in\NodeIn(a)\}/W_e(a)$
            for every $c$.\\
            Set $b_{t_\ell}(a)\leftarrow
            \beta\indi\{a=j_e\}/W_e(j_e)$ if $\barS\ne\emptyset$;
            otherwise set $b_{t_\ell}(a)\leftarrow0$.\\
        }
    }
    Set $W_{e+1}(a)=\widehat w_{e+1}(a)+3\gamma/2$ for every $a$.\\
    If $\barS\ne\emptyset$, set
    $j_{e+1}\in\argmax_{a\in\barS}\widehat m_{e+1}(a)$.\\
    Set $s_{e+2,c}=p_{eL,c}$ using \eqref{eq:ftrl-update} for every $c$.\\
}
\end{algorithm}

\begin{figure*}[!t]
  \centering
  \includegraphics[width=\textwidth]{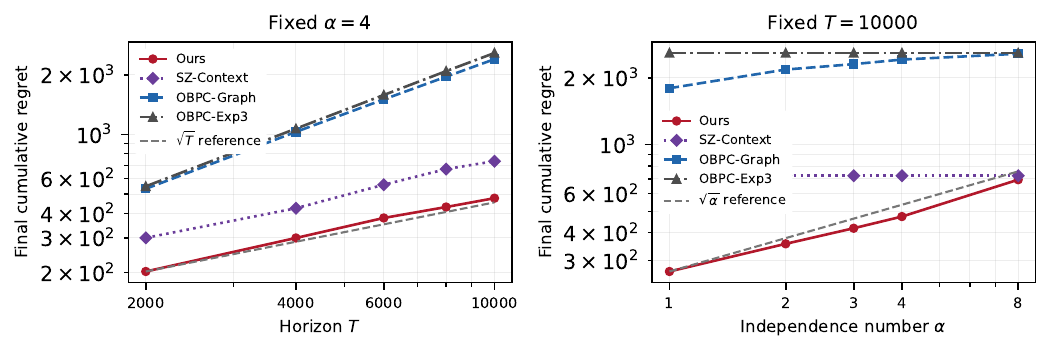}
  \Description{Two log-log plots of final cumulative regret. The left varies
  the horizon while fixing the graph independence number at four. The right
  varies the independence number while fixing the horizon at ten thousand.
  Both compare the proposed algorithm with graph-only and contextual Exp3
  baselines and include a normalized square-root reference.}
  \caption{Controlled regret scaling.  Left: \(T\) varies at fixed
  \(\alpha=4\).  Right: \(\alpha\) varies at fixed \(T=10^4\).  Dashed gray
  curves are square-root references normalized at the smallest point; markers
  show means over ten runs.  Every graph contains arm 0 without a self-loop.}
  \label{fig:scaling}
\end{figure*}

\section{Analysis}\label{sec:analysis}

The analysis combines the epoch construction of \citet{Sch23}, the graph-sum
inequality of \citet{GraphAlon15}, and a special-arm mechanism inspired by
\citet{LuoGraph23}.  The main result is the following theorem.

\begin{reptheorem}[Formal]{thm:main}
Choose
\[
\iota=C\log(MKT),\qquad
L=C_L\sqrt{\alpha T\iota},\qquad
\gamma=C_\gamma\frac{\iota}{L},\qquad
\eta=\frac{c_\eta}{L},
\]
and set $\rho=C_\rho\gamma$ and $\beta=C_\beta\gamma$, where the absolute
constants satisfy the stability conditions in
\Cref{lem:s_smaller_p_smaller_2s}.  Then
\Cref{alg:unknown} satisfies
\[\Reg_T(\pi)=\widetilde O(\sqrt{\alpha T}).\]
The hidden factors are logarithmic in $M,K,T$; in particular, the bound has no
polynomial dependence on $M$.
\end{reptheorem}

We organize the proof around two analytical proxies and a six-term regret
decomposition.  The full bounds appear in the appendix.

\paragraph{Pair reduction.}
The two rounds in each pair use the same policy, and the loss-estimation side
is chosen uniformly after both rounds are played.  Hence the expected regret
of a pair is exactly twice the expected regret of its selected loss-estimation
round:
\[
\E[R_t+R_{t+1}]=2\E[R_{t_\ell}].
\]
Here $C_s$ is the physical context and
$R_s=\ip{q_{s,C_s}-\pi_{C_s},\ell_{s,C_s}}$.  Since the policy table is
predictable and $C_t\sim\nu$,
\[
\E[R_t]=\E\!\left[\sum_{c\in[M]}\nu(c)
\ip{q_{t,c}-\pi_c,\ell_{t,c}}\right].
\]
Thus $c$ below is a counterfactual evaluation index, whereas $B_t$ is generated
by the physical loss-side context; cross-learning makes the retained loss at
$B_t$ available for every such $c$.  We drop the warm-up cost $L$, index only
selected rounds by $t$, and let $\ell_{t,c}$ denote their doubled loss.

\paragraph{Population proxies.}
The empirical denominator $W_e(a)$ is estimated from only $\Theta(L)$ contexts
and is reused throughout epoch $e$.  A pathwise comparison with the population
probability is too costly, so the proof instead controls its first two moments.
Write $u_e(a)=w_e(a)+\gamma$ and define the uncomputable pseudo-estimator
\[
\widetilde{\ell}_{t,c}(a)
=\frac{\ell_{t,c}(a)}{u_e(a)}
\indi\{B_t\in\NodeIn(a)\}.
\]
It has the same retained-feedback numerator as $\widehat\ell_{t,c}(a)$ but uses
the population denominator $u_e(a)$ instead of $W_e(a)$.

The pessimistic correction also depends on $W_e(j_e)$.  We therefore define the
fully ideal FTRL input
\[
\widetilde y_{t,c}(a)=\widetilde\ell_{t,c}(a)
+\frac{\beta}{u_e(j_e)}
\indi\{a=j_e\},
\]
where the second term is absent when no special arm exists.  Within epoch $e$,
the ideal trajectory starts from the same state as the actual trajectory and
evolves according to
\[
\widetilde p_{t,c}\propto p_{(e-1)L,c}\circ
\exp\!\left(-\eta
\sum_{s\in\Time_e^\ell,\,s<t}\widetilde y_{s,c}\right).
\]
After conditioning on the snapshot, selector sample, and randomness used for
loss estimation, $\widetilde p$ is fixed and the frequency sample remains
fresh.  This is the conditional independence needed for the denominator moment
bounds.  The actual trajectory $p$, by contrast, depends on all coordinates of
that shared sample and must be compared with $\widetilde p$ jointly.

\paragraph{Regret decomposition.}
Write $\sum_\ell=\sum_{e=2}^E\sum_{t\in\Time_e^\ell}$ and let
$\Reg_{\mathrm{main}}(\pi)$ denote the post-warm-up regret.  The pair
reduction, marginalization over the physical context, and addition and
subtraction of the two proxies give
\begin{align*}
\Reg_{\mathrm{main}}(\pi) &=\E\left[\sum_\ell
\sum_{c\in[M]}\nu(c)\ip{q_{t,c}-\pi_c,\ell_{t,c}}\right] \\
&=\underbrace{\E\left[\sum_\ell\sum_c\nu(c)
\ip{q_{t,c}-p_{t,c},\ell_{t,c}}\right]}_{\bias{1}} \\
&+\underbrace{\E\left[\sum_\ell\sum_c\nu(c)
\ip{p_{t,c}-\pi_c,\ell_{t,c}-\widetilde\ell_{t,c}}\right]}_{\bias{2}} \\
&+\underbrace{\E\left[\sum_\ell\sum_c\nu(c)
\ip{\widetilde p_{t,c}-\pi_c,\widetilde\ell_{t,c}-\widehat\ell_{t,c}}\right]}_{\bias{3}} \\
&+\underbrace{\E\left[\sum_\ell\sum_c\nu(c)
\ip{p_{t,c}-\widetilde p_{t,c},\widetilde\ell_{t,c}-\widehat\ell_{t,c}}\right]}_{\bias{4}}\\
&+\underbrace{\E\left[\sum_\ell\sum_c\nu(c)
\ip{p_{t,c}-\pi_c,\widehat\ell_{t,c}+b_t}\right]}_{\mathbf{FTRL}}\\
&\quad-\underbrace{\E\left[\sum_\ell\sum_c\nu(c)
\ip{p_{t,c}-\pi_c,b_t}\right]}_{\bias{5}}.
\end{align*}

\paragraph{What each term controls.}
$\bias{1}$ pays for uniform mixing and the snapshot fallback in the played
distribution.  Snapshot stability implies $q_t=r_t^p$ on the good event, so
this term is $O(\rho T)$ plus the negligible failure contribution.
$\bias{2}$ replaces the true loss by the population-denominator estimator.
For regular arms, its implicit-exploration drift is bounded by
$\gamma$ times the graph sum.  On the special arm, the remaining positive
drift is deliberately left for cancellation by $-\bias{5}$.

$\bias{3}$ replaces population denominators by empirical denominators while
holding the ideal trajectory fixed.  Under the conditioning above, its
coefficients are fixed and scalar first-moment bounds for $u_e(a)/W_e(a)$ apply.
$\bias{4}$ pays for replacing $\widetilde p$ by the empirical trajectory $p$.
All denominators use the same frequency contexts and may be correlated, so a
coordinate-wise comparison would not suffice.  A joint Softmax product-stability
inequality reduces this term to a sum of squared coordinate errors; scalar
second-moment bounds and linearity of expectation then finish the estimate
without assuming independence across arms.  The same comparison incorporates
$\beta/W_e(j_e)-\beta/u_e(j_e)$ on the special coordinate.

\paragraph{Closing the special-arm terms.}
The preceding steps extend the epoch analysis of \citet{Sch23}, but the usual
graph sum does not cover a dominant no-self-loop arm.  For
$a\in\barS$, strong observability gives
$\NodeIn(a)=[K]\setminus\{a\}$, so a large playing probability corresponds to a
small observation denominator.  The selector isolates the only arm for which
this can occur as $j_e$; all other arms satisfy the regular graph-sum bound.
On $j_e$, the negative correction term $-\bias{5}$ cancels the first-order
drift retained from $\bias{2}$ and $\bias{3}$ when $\beta$ is a sufficiently
large constant multiple of $\gamma$.

It remains to control the correction inside the FTRL regret.  A direct square
would produce a penalty proportional to $W_e(j_e)^{-2}$.  We instead use
translation invariance of
negative-entropy FTRL and subtract the full special-coordinate input from every
coordinate.  The special coordinate then contributes zero to the local norm.
The remaining quadratic term is weighted by the probability mass away from
$j_e$.  Writing $\bar p_t(a)=\E_{c\sim\nu}[p_{t,c}(a)]$, snapshot stability and
strong observability imply
\[
1-\bar p_t(j_e)
=\E_{c\sim\nu}\!\left[p_{t,c}(\NodeIn(j_e))\right]
\lesssim w_e(j_e)+\rho
\lesssim u_e(j_e).
\]
Because the retained-observation indicator has mean
$w_e(j_e)\le u_e(j_e)$ and $\beta\asymp\gamma\le u_e(j_e)$, this away-mass
relation removes the inverse-square dependence.  The regular local-norm terms
contribute $\widetilde O(\eta\alpha T)$.  With the parameters in
\Cref{thm:main}, the complete decomposition is therefore
$\widetilde O(\sqrt{\alpha T})$; the detailed bounds appear in the appendix.

\section{Experiments and Formalization}\label{sec:experiments}

We test the predicted dependence on the horizon \(T\) and graph independence
number \(\alpha\) under fixed asymptotic parameter rules, without tuning
individual points.  The four methods form a feedback-structure comparison.
Our method uses graph feedback and cross-learning.  SZ-Context is the
Schneider--Zimmert algorithm~\citep{Sch23}, specialized to singleton feedback
neighborhoods, and therefore uses cross-learning but ignores the graph.
OBPC-Graph maintains one graph-aware learner per context and omits
cross-learning; OBPC-Exp3 uses neither information-sharing structure.

\paragraph{Protocol.}
We use \(K=9\) arms and \(M=1{,}000\) uniformly distributed contexts.  Starting
from a directed ring of radius \(r\), we remove arm 0's self-loop and add an edge
to arm 0 from every other arm.  The graph is strongly observable, but arm 0 is
observed only when another arm is played.  Radii \(r=8,3,2,1,0\) yield
\(\alpha=1,2,3,4,8\), respectively.  We fix one oblivious, nonstationary loss
table
\[
\ell_{t,c}(a)=\operatorname{clip}_{[0,1]}
\left(B_{c,a}+0.1\sin(2\pi(t+1)/2000+\phi_c+\psi_a)\right),
\]
where \(B_{c,a}\stackrel{\mathrm{i.i.d.}}{\sim}\mathrm{Unif}[0.15,0.85]\) and
\(\phi_c,\psi_a\stackrel{\mathrm{i.i.d.}}{\sim}\mathrm{Unif}[0,2\pi)\).
For our method, $L$ is the nearest even divisor of $T$ to
$\sqrt{\alpha T/\log K}$ and
$\eta=\gamma=\rho=\beta=1/\sqrt{\alpha T}$.  For SZ-Context, the corresponding
rule is $L\approx\sqrt{KT/\log K}$ and
$\eta=\gamma=1/\sqrt{KT}$; its paired rounds, delayed snapshots, rejection
sampling, and frequency denominators follow \citet{Sch23}.

The horizon sweep fixes \(\alpha=4\) and varies
\(T\in\{2,4,6,8,10\}\times10^3\).  The graph sweep fixes \(T=10^4\) and
varies \(\alpha\in\{1,2,3,4,8\}\).  Each point averages ten runs.  Within a run,
all methods share the context sequence; seeds are reused across each sweep, and
shorter runs are prefixes of the corresponding longer runs.

\paragraph{Results.}
In \Cref{fig:scaling}, our method has the lowest final regret at every tested
\(T\) and \(\alpha\).  Its fitted log--log slopes are \(0.53\) in \(T\) and
\(0.45\) in \(\alpha\), close to the predicted exponent \(1/2\).  These
descriptive finite-sample fits do not prove the asymptotic rate, but the trend
persists across both sweeps.  SZ-Context has horizon slope $0.57$ and, as
expected for a graph-agnostic method, is constant across the graph sweep.  At
$T=10^4,\alpha=4$, our regret is $476.33$, compared with $732.77$ for
SZ-Context, $2403.75$ for OBPC-Graph, and $2582.93$ for OBPC-Exp3.  As
$\alpha$ increases from $1$ to $8$, our regret rises from $267.61$ to $694.40$;
SZ-Context remains at $723.53$.  Thus graph feedback gives a substantial gain
over cross-learning alone for denser graphs, with the gap narrowing at the
sparsest endpoint.  Complete results and reproducibility details are in
\Cref{app:experiments}.

\paragraph{Formalization and artifacts.}
The GitHub repository contains the end-to-end Lean development,
exceeding 100,000 lines of Lean code, together with the experiment code, raw
runs, and scripts used to generate \Cref{fig:scaling}; see the
\artifactlink\ (\artifacturl).

\section{Conclusion}
We resolve the open, context-independent regret question for cross-learning
contextual bandits with graphical feedback.  The result holds for oblivious
adversarial losses and every fixed strongly observable graph, including
no-self-loop arms, and gives expected regret
\(\widetilde O(\sqrt{\alpha T})\), with only logarithmic dependence on the
number of contexts.  The
special-arm selector, pessimistic correction, and loss-shifted FTRL analysis
jointly control the observation denominator that cross-learning alone cannot
repair.  Controlled experiments show the benefit of combining the two feedback
structures and are consistent with the predicted scaling in \(T\) and
\(\alpha\).  \formalizationstatement\ Extending the result to weakly observable,
unknown, or time-varying graphs remains open.

\bibliographystyle{icml2025}
\bibliography{ICML_bib}

\clearpage
\appendix
\section*{AI Usage Statement}

We used GPT-5.6 to revise the manuscript's writing and to assist with the Lean
formalization of the main theorem.

\section{Probability Space, Filtrations, and Conventions}

Losses and the graph are deterministic.  We realize the random split in
\Cref{alg:unknown} on the following product space.  For pair $r$ in epoch $e$,
let $C^f_{e,r}$ and $C^\ell_{e,r}$ denote the frequency- and loss-side contexts.
Since the two physical contexts are i.i.d. and the
permutation is uniform and independent, the pairs
$(C^f_{e,r},C^\ell_{e,r})$ are i.i.d. with law $\nu\otimes\nu$.  Conditional on
$C^\ell_{e,r}=c$, rejection sampling generates a retained output
$B^\ell_{e,r}$ with
\[
 \Pr(B^\ell_{e,r}=a\mid C^\ell_{e,r}=c)
 =\frac12r_{e,c}(a),
 \qquad
 \Pr(B^\ell_{e,r}=\bot\mid C^\ell_{e,r}=c)=\frac12.
\]
This output law is independent of the proposal $q_{t,c}$; the fallback in
\Cref{alg:unknown} merely guarantees that the rejection probability is valid.
Thus one may equivalently expose the frequency contexts and the loss-side
pairs $(C^\ell,B^\ell)$ first under this product law, and then generate the
proposal action from its conditional law to recover the original joint
distribution.  Only $(C^\ell,B^\ell)$ enters the ideal loss numerator; the
proposal action is not included in the conditioning used for denominator
moments.

After taking expectation over the physical context, we write regret and all
bias terms explicitly as the population sum $\sum_c\nu(c)$.  Here $c$ is only
a counterfactual evaluation index: $B_t$ and $I_{t,a}$ are generated from the
physical loss-side pair $(C^\ell,B^\ell)$, while cross-learning supplies the
retained loss coordinate for every fixed $c$.  Below,
$\E_{c\sim\nu}[f(c)]$ is only shorthand for
$\sum_{c\in[M]}\nu(c)f(c)$; it does not introduce an additional algorithmic
draw.

Let $\mathcal F_t$ contain the chronological physical history through round
$t$.  For a fixed epoch $e\ge2$, let $\mathcal D_e$ be the sigma-field generated
by the frequency contexts $\{C^f_{e-1,r}\}_r$ used to form $\widehat w_e$, and
let $\mathcal G_e$ contain
\begin{enumerate}
    \item the history strictly preceding that frequency block, including the
    FTRL state at the start of epoch $e$ and the delayed snapshot $s_e$;
    \item all loss-side contexts and retained outputs needed to form $j_e$ and
    the ideal loss numerators in epoch $e$; and
    \item all frequency blocks other than $\mathcal D_e$.
\end{enumerate}
The preceding product construction gives $\mathcal D_e\perp\mathcal G_e$
conditionally on the history before epoch $e-1$.  In particular, $j_e$, the
population denominators $u_e$, and the fully ideal within-epoch trajectory are
$\mathcal G_e$-measurable, whereas $W_e$ and the vector of ratio errors defined
below are $\mathcal D_e$-measurable.  Every use of a scalar denominator moment
bound is conditional on $\mathcal G_e$ and the pre-block history.  The
coordinates of $W_e$ are generally dependent because they use the same
frequency contexts; no independence across arms is assumed.

Following the pair reduction in \Cref{sec:analysis}, we work only with the
selected loss-estimation round and use $\ell_{t,c}$ for its doubled loss.  All
FTRL inequalities are pathwise after the relevant trajectory is fixed.

The following concentration inequality controls the good-event failures used
in Biases 1--4.

\begin{lemma}[Freedman's Inequality]
Fix any $\lambda>0$ and $\delta \in(0,1)$. Let $X_t$ be a random process with respect to a filtration $\mathcal{F}_t$ such that $\mu_t = \mathbb{E}\left[X_t \mid \mathcal{F}_{t-1}\right]$ and $V_t = \mathbb{E}\left[X_t^2 \mid \mathcal{F}_{t-1}\right]$, and satisfying $\lambda X_t \le 1$. Then, with probability at least $1-\delta$, we have for all $t$,
$$
\sum_{s=1}^t X_s - \mu_s \le \lambda \sum_{s=1}^t V_s + \frac{\log(1 / \delta)}{\lambda}.
$$
\end{lemma}

\section{Graph Structure and Special-Arm Selection}

The graph inverse lemma controls sums over regular arms, while the selector
isolates the only no-self-loop coordinate whose denominator can be small.

\begin{lemma}[Lemma 5, \citet{GraphAlon15}]
\label{lem:graph-inverse}
Let $G=(V, E)$ be a directed graph with $|V|=K$, in which each node $i \in V$ is assigned a positive weight $w_i$. Assume that $\sum_{i \in V} w_i \le 1$, and that $w_i \geq \epsilon$ for all $i \in V$ for some constant $0<\epsilon<\frac{1}{2}$. Then
\[
\sum_{i \in V} \frac{w_i}{w_i+\sum_{j \in N^{\operatorname{in}(i)}} w_j} \le 4 \alpha \ln \frac{4 K}{\alpha \epsilon}
\]
where $\alpha=\alpha(G)$ is the independence number of $G$.
\end{lemma}

\begin{lemma}[Exponential-weights stability]
\label{lem:exp3-update}
This is Lemma~11 of \citet{Sch23}.
Let $z\in[-\frac{1}{2},\frac{1}{2}]^K$, $x\in\Delta([K])$ and $\widetilde x\propto x\circ\exp(z)$. Then for all $k\in[K]$:
\begin{align*}
    x_k \exp(z_k - 2\ip{x,|z|})\le \widetilde x_k \le x_k \exp(z_k + \ip{x,|z|})\,.
\end{align*}
\end{lemma}

\section{Good Events and the Clipped/Stopped Proxy}

We now define the proxy process used below.  For each epoch $e$, let
\begin{align*}
\mathcal E^F_e
&=\left\{\forall a\in[K]:
  |\widehat w_e(a)-w_e(a)|\le
  2\max\left\{\sqrt{\frac{w_e(a)\iota}{L}},
  \frac{\iota}{L}\right\}\right\},\\
\mathcal E^L_e
&=\left\{\max_{c\in[M],a\in[K]}
  \sum_{t\in\Time_e^\ell}
  \frac{I_{t,a}\ell_{t,c}(a)}{u_e(a)}
  \le L+\frac{\iota}{\gamma}\right\},\\
\mathcal E^S_e
&=\left\{\max_{a\in[K]}|\widehat m_e(a)-m_e(a)|
  \le\sqrt{\frac{\iota}{L}}\right\}.
\end{align*}
Here $I_{t,a}=\indi\{B_t\in\NodeIn(a)\}$ and
$u_e(a)=w_e(a)+\gamma$.  The events $\mathcal E^L_e$ and
$\mathcal E^S_e$ are $\mathcal G_e$-measurable, while
$\mathcal E^F_e$ is measurable with respect to
$\mathcal G_e\vee\mathcal D_e$.
For a round $t$, let $\mathcal E^L_{e,<t}$ denote the same loss event with
the sum restricted to $s<t$.  It is measurable before the loss-side draw at
round $t$.

Let
\[
B_e=C\left(L+\frac{\iota}{\gamma}
             +\frac{L\beta}{\gamma}\right)
\]
for a sufficiently large absolute constant $C$.  For a loss-estimation round
$t\in\Time_e^\ell$, define the non-frequency prefix event
\[
\mathcal A_{e,t}=\mathcal E^S_e\cap\mathcal E^L_{e,<t}\cap
\left\{\max_{c,a}
 \sum_{s\in\Time_e^\ell:\,s<t}\widetilde y_{s,c}(a)
 \le B_e\right\}
\cap\{\eta B_e\le c_0\},
\]
where $c_0>0$ is the absolute constant in the Softmax stability lemma.  Set
\[
J_{e,t}=\indi\!\left\{\mathcal A_{e,s}
  \text{ holds for every }s\in\Time_e^\ell,\ s\le t\right\}.
\]
Thus $J_{e,t}$ is predictable, $\mathcal G_e$-measurable, and absorbing: once
it is zero, it stays zero for the rest of the epoch.  Crucially,
$\mathcal E^F_e$ is not part of $J_{e,t}$.

For every coordinate define
\[
\begin{aligned}
\theta_e(a)&=\frac{u_e(a)}{W_e(a)},\\
\breve\theta_e(a)&=\min\{2,\max\{1/2,\theta_e(a)\}\},&
\phi_{e,a}&=1-\breve\theta_e(a).
\end{aligned}
\]
Hence $\breve\theta_e(a)\in[1/2,2]$ and
$\phi_{e,a}\in[-1,1/2]$ on every sample path.  Let
$\widetilde y_{t,c}$ denote the stopped ideal input defined below (it already
contains $J_{e,t}$).  The clipped, stopped proxy input is
\[
y_{t,c}(a)=\breve\theta_e(a)\widetilde y_{t,c}(a).
\]
Both trajectories start from the real FTRL state at the beginning of the
epoch and are frozen after the first round with $J_{e,t}=0$.  Conditional on
$\mathcal G_e$, the entire stopped ideal trajectory and $J_{e,t}$ are fixed,
while the only remaining current-denominator randomness in the proxy is the
shared vector $\phi_e$, measurable with respect to $\mathcal D_e$.

Let $Q=\bigcap_e(\mathcal E^F_e\cap\mathcal E^L_e\cap
\mathcal E^S_e)$, together with the deterministic stability inequalities in
the statements below.  Bernstein, Hoeffding, and a union bound give
\[
\Pr(Q) \ge 1-\frac{T}{L}(4K+MK)\exp(-\iota).
\]
On $Q$, denominator concentration gives
$\theta_e(a)\in[1/2,2]$, so clipping is inactive; the prefix bounds give
$J_{e,t}=1$.  An induction over the loss-estimation rounds then shows that the
proxy and real FTRL inputs, distributions, and estimators coincide on $Q$.
Writing $\Reg_T^\dagger$ for the real algorithm and $\Reg_T$ for the proxy,
we obtain
\[
\left|\Reg_T(\pi) - \Reg_T^{\dagger}(\pi)\right|
\le 2T\indi\{Q^c\},
\]
so the difference in expectation is $O(1)$.  From now on, all quantities
without a dagger are the clipped, stopped proxy quantities.

\subsection{Special-arm selector properties}

The following two lemmas are the only new ingredients that are truly specific to the no-self-loop case. All other tools (random denominators, counterfactual FTRL, snapshot stability) are inherited from the self-loop analysis.

\begin{lemma}[Selector concentration]
\label{lem:selector}
If $\max_{a \in [K]} |\widehat{m}_e(a) - m_e(a)| \le \varepsilon$, then for every $a \in \barS \setminus \{j_e\}$,
\[
m_e(a) \le \frac12 + \varepsilon
\quad \text{and} \quad
w_e(a) \ge \frac14 - \frac{\varepsilon}{2}.
\]
\end{lemma}

\begin{proof}
By the empirical maximality of $j_e$, we have $m_e(a) \le m_e(j_e) + 2\varepsilon$. Since $a$ and $j_e$ are two different coordinates of the probability vector $r_{e,c}$, we also have $m_e(a) + m_e(j_e) \le 1$. Adding the two inequalities gives $m_e(a) \le 1/2 + \varepsilon$. For $a \in \barS$, strong observability gives $w_e(a) = \frac12(1 - m_e(a))$, which yields the second claim.
\end{proof}

\begin{lemma}[Special-arm odds]
\label{lem:special-odds}
Under the stability condition
\[
\gamma \ge C\frac{\iota}{L}, \qquad \rho \asymp \beta \asymp \gamma, \qquad
\eta C\left(L + \frac{\iota}{\gamma} + \frac{L\beta}{\gamma}\right) \le c,
\]
if $j_e$ exists, then for the clipped trajectory and the fully ideal counterfactual trajectory defined below,
\[
J_{e,t} \frac{1 - \bar{p}_t(j_e)}{u_e(j_e)} \le C,
\qquad
J_{e,t} \frac{1 - \bar{\widetilde p}_t(j_e)}{u_e(j_e)} \le C,
\]
where $\bar{p}_t(a) = \E_{c \sim \nu}[p_{t,c}(a)]$, $\bar{\widetilde p}_t(a)=\E_{c\sim\nu}[\widetilde p_{t,c}(a)]$, and $u_e(a) = w_e(a) + \gamma$.
\end{lemma}

\begin{proof}
If $J_{e,t}=0$, the claim is trivial. On active rounds, snapshot stability gives both $p_{t,c}(a) \ge s_{e,c}(a)/2$ and $\widetilde p_{t,c}(a)\ge s_{e,c}(a)/2$ for all $a$. If $m_e(j_e) \le 1/2$, then $u_e(j_e) \ge w_e(j_e) \ge 1/4$. If $m_e(j_e) > 1/2$, then $w_e(j_e) = \frac12(1 - m_e(j_e))$ and snapshot stability implies that both $1 - \bar{p}_t(j_e)$ and $1-\bar{\widetilde p}_t(j_e)$ are at most $C(w_e(j_e) + \rho) \le C u_e(j_e)$ since $\rho \asymp \gamma$. Both cases give the desired bounds.
\end{proof}

\section{Random-Denominator Bounds}

For each regular arm $a \in \mathcal{B}_e$, let $\widehat{w}$ be the empirical mean of $n = \Theta(L)$ i.i.d.\ random variables bounded in $[0,1/2]$ with mean $w$, and define
\[
u = w + \gamma, \qquad W = \widehat{w} + \frac32 \gamma, \qquad \theta = \frac{u}{W}, \qquad g = 1 - \theta.
\]
Let $\phi$ be the coordinate-wise clipped version of $g$ with $-1 \le \phi \le 1/2$ and $\breve{\theta} = 1 - \phi \in [1/2,2]$. The following bounds hold (see \citet{Sch23}):
\begin{align}
-Ce^{-\iota} \le \E[\phi \mid \mathcal G_e] \le C\frac{\gamma}{u} + C e^{-\iota}, \label{eq:ratio-first} \\
\E[\phi^2 \mid \mathcal G_e] \le C\left(\frac{w}{L u^2} + \frac{\gamma^2}{u^2}\right) + C e^{-\iota}. \label{eq:ratio-second}
\end{align}
The same bounds hold for every coordinate, including the data-dependent special arm, because the selector sample is independent of the frequency sample.

\section{Snapshot and Joint Trajectory Stability}

The following lemma is the analogue of Lemma~9 in \citet{Sch23} for the graphical-feedback setting.

\begin{lemma}
\label{lem:s_smaller_p_smaller_2s}
Suppose
\begin{equation}
\begin{gathered}
\gamma\ge C\frac{\iota}{L},\qquad
0<\rho\le C\gamma,qquad 0\le\beta\le C\gamma,\\
\eta C\left(L+\frac{\iota}{\gamma}
                 +\frac{L\beta}{\gamma}\right)\le\log 2.
\end{gathered}
\label{eq:snapshot-stability-condition}
\end{equation}
Then under $Q$, for all $t\in\Time_e$, $a\in[K]$, and $c\in[M]$
simultaneously,
\[
2 s_{e,c}(a) \ge p_{t,c}(a) \ge s_{e,c}(a) / 2
\quad \text{and} \quad
2 s_{e,c}(a) \ge \widetilde{p}_{t,c}(a) \ge s_{e,c}(a) / 2.
\]
Since $p_{t,c}(a)\ge s_{e,c}(a)/2$, mixing both sides gives
\[
r^p_{t,c}(a)=(1-\rho)p_{t,c}(a)+\rho/K
\ge \frac12\bigl((1-\rho)s_{e,c}(a)+\rho/K\bigr)
=\frac12r_{e,c}(a).
\]
Thus the pointwise rejection test succeeds and $q_t=r^p_t$ for all $t\in\Time_e$.
\end{lemma}

\begin{proof}
Fix $e,t,c$, and let $H_{t,c}(a)$ be the cumulative FTRL input to coordinate
$a$ between the time at which $s_{e,c}$ is taken and round $t$.  The two-epoch
delay means that this interval intersects at most two complete epochs and one
current prefix.  On $Q$, \(\mathcal E^L\), denominator stability, and
$\breve\theta\le2$ give, uniformly over $a,c$,
\begin{equation}
0\le H_{t,c}(a)
\le C\left(L+\frac{\iota}{\gamma}
              +\frac{L\beta}{\gamma}\right)=:B.
\label{eq:cumulative-input-bound}
\end{equation}
Indeed, the first two terms bound the retained-loss inputs, while there are at
most $L$ selected updates and each special-arm correction is at most
$C\beta/\gamma$.  The same bound, without the factor from
$\breve\theta$, holds for the ideal cumulative input.  The assumptions
$\gamma\ge C\iota/L$ and $\beta\le C\gamma$ imply $B\le C'L$.

Negative-entropy FTRL has the exact multiplicative form
\[
\frac{p_{t,c}(a)}{s_{e,c}(a)}
=\frac{\exp(-\eta H_{t,c}(a))}
{\sum_b s_{e,c}(b)\exp(-\eta H_{t,c}(b))}.
\]
Since every $H_{t,c}(b)$ lies in $[0,B]$, both the numerator and normalizer lie
between $e^{-\eta B}$ and $1$.  Hence the displayed ratio lies in
$[e^{-\eta B},e^{\eta B}]\subseteq[1/2,2]$ by
\eqref{eq:snapshot-stability-condition}.  Applying the same calculation to
the ideal inputs proves both factor-two bounds.  The final rejection-sampling
claim follows from the mixing calculation in the lemma statement.
\end{proof}

Applying \Cref{lem:graph-inverse} to the self-loop arms and \Cref{lem:selector} to the remaining no-self-loop arms, we obtain
\[
J_{e,t} \sum_{a \in \mathcal{B}_e} \frac{\bar{p}_t(a)}{u_e(a)} \le \Agraph = O\!\left(\alpha \log \frac{K^2}{\alpha \rho}\right) + O(1).
\]
The same inequality holds with $p$ replaced by $\widetilde p$. We call these inequalities the graph-sum bound.

\subsection{Fully ideal trajectory and joint stability}

For every active loss-estimation round, define the ideal input
\[
\widetilde y_{t,c}(a)=
\begin{cases}
J_{e,t}I_{t,a}\ell_{t,c}(a)/u_e(a),&a\in\mathcal B_e,\\
J_{e,t}(I_{t,j_e}\ell_{t,c}(j_e)+\beta)/u_e(j_e),&a=j_e,
\end{cases}
\]
where the second line is absent when there is no special arm. Let $\widetilde p_{t,c}$ start from the same FTRL state as $p_{t,c}$ at the beginning of epoch $e$ and feed $\widetilde y_{s,c}$ throughout that epoch. The proxy trajectory instead feeds
\[
y_{t,c}(a)=\breve\theta_{e,a}\widetilde y_{t,c}(a),
\qquad \phi_{e,a}=1-\breve\theta_{e,a}.
\]
Conditional on $\mathcal G_e$, $\widetilde p$ is fixed and the frequency block
$\mathcal D_e$ remains fresh.  The coordinates $\phi_{e,a}$ may be mutually
dependent; the proof does not split samples across coordinates.

\begin{lemma}[Joint product stability]
\label{lem:joint-stability}
Under \eqref{eq:snapshot-stability-condition}, on every active round,
\[
\sum_{a\in[K]}|p_{t,c}(a)-\widetilde p_{t,c}(a)|\,|\phi_{e,a}|
\le C\sum_{a\in[K]}\widetilde p_{t,c}(a)\phi_{e,a}^2.
\label{eq:joint-product}
\]
If a special arm $j=j_e$ exists, then also
\begin{align}
|p_{t,c}(j)-\widetilde p_{t,c}(j)|
&\le C\eta L(1-\widetilde p_{t,c}(j))|\phi_{e,j}|\\
&\quad+C\eta L\sum_{a\ne j}\widetilde p_{t,c}(a)|\phi_{e,a}|,
\label{eq:joint-special-first}
\end{align}
The consequences that use the special-arm odds bound hold only after averaging
over the context:
\begin{align}
  \E_{c\sim\nu}\!\left[|p_{t,c}(j)-\widetilde p_{t,c}(j)|
  |\phi_{e,j}|\right]
&\le C\eta L u_e(j)\phi_{e,j}^2\\
&\quad+C\eta L\sum_{a\ne j}
  \bar{\widetilde p}_t(a)\phi_{e,a}^2,
\label{eq:joint-special-second}\\
  \E_{c\sim\nu}\!\left[|p_{t,c}(j)-\widetilde p_{t,c}(j)|\right]
&\le C\eta L u_e(j).
\label{eq:joint-special-crude}
\end{align}
\end{lemma}

\begin{proof}
Within epoch $e$, write
\[
z_{t,c}(a)=\eta\phi_{e,a}\sum_{s\in\Time_e^\ell,\,s<t}\widetilde y_{s,c}(a).
\]
The identity $y=(1-\phi)\widetilde y$ and the common initial state give
$p_{t,c}\propto\widetilde p_{t,c}\circ\exp(z_{t,c})$.  By
\eqref{eq:cumulative-input-bound},
\begin{equation}
|z_{t,c}(a)|\le
C\eta\left(L+\frac{\iota}{\gamma}
              +\frac{L\beta}{\gamma}\right)|\phi_{e,a}|
\le c_0|\phi_{e,a}|,
\label{eq:z-phi-bound}
\end{equation}
where $c_0\le1/4$ after decreasing the absolute constant in the stability
condition.  In particular $z\in[-1/2,1/2]^K$.  Applying
\Cref{lem:exp3-update} in both directions, and using
$|e^x-1|\le2|x|$ for $|x|\le1$, gives
\[
|p_{t,c}(a)-\widetilde p_{t,c}(a)|
\le C\widetilde p_{t,c}(a)
\left(|\phi_{e,a}|+\ip{\widetilde p_{t,c},|\phi_e|}\right).
\]
Multiplying by $|\phi_{e,a}|$ and summing over $a$ gives
\[
\sum_a|p_{t,c}(a)-\widetilde p_{t,c}(a)||\phi_{e,a}|
\le C\left(\sum_a\widetilde p_{t,c}(a)\phi_{e,a}^2
+\ip{\widetilde p_{t,c},|\phi_e|}^2\right).
\]
Jensen's inequality bounds the squared inner product by the first term and
proves \eqref{eq:joint-product}.  This argument is pathwise in the shared
vector $\phi_e$ and therefore uses no independence between its coordinates.

For the special coordinate, subtract $z_{t,c}(j)$ from every exponent, which
does not change the Softmax distribution.  The normalizer now differs from one
only through coordinates $a\ne j$.  Using
$|z_{t,c}(a)-z_{t,c}(j)|\le
C\eta L(|\phi_{e,a}|+|\phi_{e,j}|)$ in the same exponential estimate yields
\eqref{eq:joint-special-first}.  Multiply that inequality by
$|\phi_{e,j}|$ and average over $c$.  Since $\phi_e$ is shared across contexts,
$2|xy|\le x^2+y^2$ and \Cref{lem:special-odds} give
\[
\begin{aligned}
  &\E_{c\sim\nu}[|p_{t,c}(j)-\widetilde p_{t,c}(j)||\phi_{e,j}|]\\
  &\quad\le C\eta L\left(
  \E_{c\sim\nu}[1-\widetilde p_{t,c}(j)]\phi_{e,j}^2
+\sum_{a\ne j}\bar{\widetilde p}_t(a)\phi_{e,a}^2\right)\\
&\quad\le C\eta L\left(u_e(j)\phi_{e,j}^2
+\sum_{a\ne j}\bar{\widetilde p}_t(a)\phi_{e,a}^2\right),
\end{aligned}
\]
which is \eqref{eq:joint-special-second}.  Averaging
\eqref{eq:joint-special-first} directly and using
$|\phi_{e,a}|\le1$ and
$\sum_{a\ne j}\bar{\widetilde p}_t(a)\le Cu_e(j)$ proves
\eqref{eq:joint-special-crude}.  Notice that the averaged odds lemma is never
strengthened to a pointwise-in-context statement.
\end{proof}

\section{Bias 1: Playing Distribution}

This component controls explicit mixing and the rare snapshot fallback; it is
needed to pass from the implemented distribution \(q_t\) to the FTRL
distribution \(p_t\).

\begin{lemma}[Bias 1]\label{lem:bias1}
Under the good-event construction,
\[
\bias{1}=\E\!\left[\sum_{e=2}^E\sum_{t\in\Time_e^\ell}
\sum_{c\in[M]}\nu(c)\langle q_{t,c}-p_{t,c},\ell_{t,c}\rangle\right]
\le \rho T+O(1).
\]
\end{lemma}
\begin{proof}
On \(Q\), snapshot stability implies \(q_{t,c}=r^p_{t,c}\), and
\[
\langle r^p_{t,c}-p_{t,c},\ell_{t,c}\rangle
=\rho\left(K^{-1}\sum_a\ell_{t,c}(a)
-\langle p_{t,c},\ell_{t,c}\rangle\right)\le2\rho.
\]
There are at most \(T/2\) selected rounds.  On \(Q^c\) the corresponding
doubled-loss difference is at most two per selected round.  The probability
bound for \(Q^c\), with proxy transfer, gives the \(O(1)\) remainder.
\end{proof}

\section{Bias 2: Population-Estimator Drift}

Bias 2 replaces true losses by the population-denominator pseudo-estimator.
We first control regular arms and then identify the special-arm drift canceled
by Bias 5.

\subsection{Regular-arm part}
\begin{lemma}[Bias 2, regular arms]\label{lem:bias2-regular}
For \(\mathcal B_e=[K]\setminus\{j_e\}\) when a special arm exists and
\(\mathcal B_e=[K]\) otherwise,
\[
\mathbf{Bias}_{2,\mathrm{reg}}\le C\gamma\Agraph T+O(1).
\]
\end{lemma}
\begin{proof}
For a selected analysis round, the retained-observation indicator is generated
from the physical loss-side context and has marginal mean $w_e(a)$.  For each
fixed evaluation coordinate $c$, cross-learning therefore gives
\[
\E[\widetilde\ell_{t,c}(a)\mid\mathcal F_{t-1}]
=\frac{w_e(a)}{u_e(a)}\ell_{t,c}(a),\qquad
u_e(a)=w_e(a)+\gamma.
\]
The learner contribution is at most
\[
\gamma\sum_{e,t}\sum_{a\in\mathcal B_e}
  J_{e,t}\frac{\E_{c\sim\nu}[p_{t,c}(a)\ell_{t,c}(a)]}{u_e(a)}
\le C\gamma\Agraph T.
\]
The comparator contribution has the opposite sign and is non-positive.
Stopped rounds and \(Q^c\) contribute \(O(1)\).
\end{proof}

\subsection{Special-arm part}
For \(j=j_e\), write
\[
\begin{aligned}
  \bar p_t&=\E_{c\sim\nu}[p_{t,c}(j)],&
  \bar\pi_e&=\E_{c\sim\nu}[\pi_c(j)],\\
  \Delta_t&=\E_{c\sim\nu}[(p_{t,c}(j)-\pi_c(j))\ell_{t,c}(j)].
\end{aligned}
\]
and \(X_t=(\bar p_t-\bar\pi_e)_+\).  The population drift is
\[
\Delta_t-\frac{w_e(j)}{u_e(j)}\Delta_t
=\frac{\gamma}{u_e(j)}\Delta_t .
\]
It need not be non-positive.  The special-arm odds lemma bounds every part
away from \(X_t\) by \(C\gamma\), while the \(X_t\) part is retained for the
Bias 5 cancellation.

\section{Bias 3: Empirical Denominator on the Ideal Trajectory}

Bias 3 compares population and empirical denominators while holding the ideal
trajectory fixed.  Conditional first-moment bounds can then be used without
an illicit independence assumption.

\subsection{Regular-arm part}
\begin{lemma}[Bias 3, regular arms]\label{lem:bias3-regular}
\[
\mathbf{Bias}_{3,\mathrm{reg}}\le
C\Agraph\left(\gamma+\frac1L\right)T+O(1).
\]
\end{lemma}
\begin{proof}
For \(a\in\mathcal B_e\),
\[
\widetilde\ell_{t,c}(a)-\widehat\ell_{t,c}(a)
=J_{e,t}\phi_{e,a}\frac{I_{t,a}\ell_{t,c}(a)}{u_e(a)}.
\]
Condition on the snapshot, selector, and loss-estimation randomness.  Then
\(\widetilde p\) and all coefficients are fixed while the frequency sample is
fresh.  Apply \eqref{eq:ratio-first} and the graph-sum bound.  For the
comparator, expanding \(\breve\theta=1-\phi\) yields a centered retention
martingale, a term controlled by the lower side of \eqref{eq:ratio-first}, and
a non-positive implicit-exploration term.  Their total is \(O(1)\).
\end{proof}

\subsection{Special-arm part}
Let \(\widetilde\Delta_t,\widetilde X_t\) be the preceding quantities with
\(p\) replaced by \(\widetilde p\), and define
\[
\widetilde V=\sum_{e:j_e\ {\rm exists}}\sum_{t\in\Time_e^\ell}
J_{e,t}\frac{\widetilde X_t}{u_e(j_e)}.
\]
The selector is fixed before the fresh frequency sample, so
\eqref{eq:ratio-first} applies to the data-dependent coordinate:
\[
\mathbf{Bias}_{3,\mathrm{sp}}\le C\gamma\E[\widetilde V]+C\gamma T+O(1).
\]
By \eqref{eq:joint-special-crude},
\[
|\bar p_t(j_e)-\bar{\widetilde p}_t(j_e)|
  \le \E_{c\sim\nu}|p_{t,c}(j_e)-\widetilde p_{t,c}(j_e)|
\le C\eta L u_e(j_e).
\]
Since $x\mapsto x_+$ is one-Lipschitz, summing this bound gives
\(\widetilde V\le V+C\eta LT\), where \(V\) uses \(X_t\).

\section{Bias 4: Trajectory Perturbation}

Bias 4 pays for replacing the ideal trajectory by the empirical one and
handles dependence between all coordinate denominators and FTRL.

\subsection{Regular-arm part}
\begin{lemma}[Bias 4, regular arms]\label{lem:bias4-regular}
\[
\mathbf{Bias}_{4,\mathrm{reg}}\le
C\Agraph\left(\gamma+\frac1L\right)T+O(1).
\]
\end{lemma}
\begin{proof}
Predictability before the current physical loss-side draw, followed by the
population-weighted sum over fixed evaluation coordinates, and
\eqref{eq:joint-product} imply
\[ 
\E\!\left[\sum_{a\in\mathcal B_e}
\E_{c\sim\nu}\!\left[|p_{t,c}(a)-\widetilde p_{t,c}(a)|\right]
\frac{I_{t,a}|\phi_{e,a}|}{u_e(a)}\right]
\le C\E\!\left[\sum_a\bar{\widetilde p}_t(a)\phi_{e,a}^2\right].
\]
Apply \eqref{eq:ratio-second} coordinate-wise and then the graph-sum bound.
No independence among the coordinates \(\phi_{e,a}\) is used.
\end{proof}

\subsection{Special-arm part}
\begin{lemma}[Bias 4, special arm]\label{lem:bias4-special}
\[
\mathbf{Bias}_{4,\mathrm{sp}}\le C\eta\Agraph(1+L\gamma)T+O(1).
\]
\end{lemma}
\begin{proof}
Condition on $\mathcal G_e$ and then take the population-weighted sum of the
special-coordinate integrand over fixed evaluation coordinates.  Equation
\eqref{eq:joint-special-second} bounds it by
\[
C\eta L\left(u_e(j_e)\phi_{e,j_e}^2+
\sum_{a\ne j_e}\bar{\widetilde p}_t(a)\phi_{e,a}^2\right).
\]
Now take expectation over the frequency block.  The special-arm odds bound,
\eqref{eq:ratio-second}, and the regular graph-sum bound yield the claim after
summing epochs and active rounds.  Only linearity of expectation is used across
coordinates.
\end{proof}

\section{Pessimistic Bias 5 and Its Cancellation}

Bias 5 is zero on regular arms.  On the special arm it cancels the first-order
drift left by Biases 2 and 3.

\begin{lemma}[Pessimistic cancellation]\label{lem:bias5}
Let \(\beta=C_0\gamma\) for a sufficiently large absolute constant \(C_0\).
Then, excluding the shifted-FTRL quadratic term,
\[
\begin{aligned}
&\mathbf{Bias}_{2,\mathrm{sp}}+\mathbf{Bias}_{3,\mathrm{sp}}
+\mathbf{Bias}_{4,\mathrm{sp}}-\bias{5}\\
&\quad\le C\!\left(\beta T+\gamma T+
\eta\Agraph(1+L\gamma)T\right)+O(1).
\end{aligned}
\]
\end{lemma}
\begin{proof}
Let \(U_t=1-\bar p_t(j_e)\).  The sign relations
\[
\Delta_t^+\le X_t+U_t,\qquad
\Delta_t^-\le U_t,\qquad
(\bar\pi_e-\bar p_t)_+\le U_t
\]
and \Cref{lem:special-odds} give \(U_t/u_e(j_e)\le C\).  Combining the
population drift with the negative fed bias gives
\[
\frac{\gamma\Delta_t-\beta(\bar p_t-\bar\pi_e)}{u_e(j_e)}
\le-\frac{c\beta X_t}{u_e(j_e)}+C\beta.
\]
Bias 3 contributes \(C\gamma\E[V]+C\gamma T+C\eta L\gamma T\), and Bias 4 is
bounded by \Cref{lem:bias4-special}.  Taking \(C_0\) large absorbs
\(C\gamma\E[V]\) into \(-c\beta\E[V]\).
\end{proof}

\section{Shifted FTRL and the Special-Arm Quadratic Term}
\label{sec:shifted-ftrl}

This pathwise lemma supplies the exact contextual FTRL endpoint and makes clear
that the played action need not be sampled from \(p_t\).

\begin{lemma}[Shifted negative-entropy FTRL]\label{lem:ftrl-shift}
Let \((\mathcal F_t)\) be any filtration, let \(p_1\) be uniform, and let
\(y_t\) be \(\mathcal F_t\)-measurable with
\[
p_{t+1,i}=\frac{p_{t,i}e^{-\eta y_{t,i}}}
{\sum_jp_{t,j}e^{-\eta y_{t,j}}}.
\]
For every fixed \(v\in\Delta_K\) and scalar \(z_t\) satisfying
\(\eta(y_{t,i}-z_t)\ge-3\),
\[
\sum_t\langle p_t-v,y_t\rangle
\le\frac{\mathrm{KL}(v\Vert p_1)}{\eta}
+2\eta\sum_{t,i}p_{t,i}(y_{t,i}-z_t)^2.
\]
In particular, \(\mathrm{KL}(v\Vert p_1)\le\log K\).  This is pathwise, so it
remains valid after conditioning on any past sigma-field and taking expectation.
\end{lemma}
\begin{proof}
Subtracting \(z_t\mathbf1\) changes neither the inner product nor the update.
The exponential-potential inequality
\[
\log\sum_i p_{t,i}e^{-\eta(y_{t,i}-z_t)}
\le-\eta\langle p_t,y_t-z_t\mathbf1\rangle
+2\eta^2\langle p_t,(y_t-z_t\mathbf1)^2\rangle
\]
holds under the stated lower bound.  Sum over \(t\), telescope the log
normalizers, apply the variational lower bound with comparator \(v\), and
divide by \(\eta\).
\end{proof}

For each context \(c\), apply the lemma to the exact update in
\Cref{alg:unknown}, \(y_{t,c}=\widehat\ell_{t,c}+b_t\), and \(v=\pi_c\);
then average with weights \(\nu(c)\).  On active rounds choose
\[
z_{t,c}=J_{e,t}\breve\theta_{e,j_e}
\frac{I_{t,j_e}\ell_{t,c}(j_e)+\beta}{u_e(j_e)}.
\]
The special coordinate contributes zero to the local norm.  The away-mass term
\[
Q_{\rm sp}=2\eta\sum_{e:j_e\ {\rm exists}}\sum_{t\in\Time_e^\ell}
  \sum_{a\ne j_e}\E_{c\sim\nu}[p_{t,c}(a)z_{t,c}^2]
\]
satisfies \(\E[Q_{\rm sp}]\le C\eta T+O(1)\) by the special-arm odds lemma.
For regular arms, predictability, retention, and the graph sum give
\[
2\eta\E\!\left[\sum_{e,t}\sum_{a\in\mathcal B_e}
  J_{e,t}\E_{c\sim\nu}[p_{t,c}(a)\widehat\ell_{t,c}(a)^2]\right]
\le C\eta\Agraph T+O(1).
\]
Thus
\[
\mathbf{FTRL}\le\frac{\log K}{\eta}
+C\eta\Agraph T+\E[Q_{\rm sp}]+O(1).
\]

\section{Assembly of the Main Theorem}

Combining the preceding lemmas in the same order as the decomposition in
\Cref{sec:algorithm} gives
\[
\begin{aligned}
\Reg_T(\pi)\le C\!\bigg[&L+\rho T+\frac{\log K}{\eta}
+\eta\Agraph T\\
&+\Agraph\left(\gamma+\frac1L\right)T+\beta T+\gamma T\\
&+\eta\Agraph(1+L\gamma)T\bigg]+O(1).
\end{aligned}
\]
Choose
\[
\begin{gathered}
\iota=C\log(MKT),\qquad L=C_L\sqrt{\alpha T\iota},\\
\gamma=C_\gamma\iota/L,\qquad \rho=C_\rho\gamma,\\
\beta=C_\beta\gamma,\qquad \eta=c_\eta/L.
\end{gathered}
\]
Since \(\Agraph=\widetilde O(\alpha)\), this is
\(\widetilde O(\sqrt{\alpha T})\).  For \(T<C\alpha\iota\), the trivial bound
completes the proof of \Cref{thm:main}.

\clearpage
\section{Extended Experiments and Reproducibility}
\label{app:experiments}

\paragraph{Complete numerical results.}
Tables~\ref{tab:horizon-scaling} and~\ref{tab:alpha-scaling} report the values
underlying \Cref{fig:scaling}.  Parentheses contain 95\% confidence-interval
half-widths, computed as $1.96$ times the standard error of the per-seed
difference between a method's loss and the shared oracle loss.  The log--log slopes in the main text are ordinary
least-squares fits to the five final-regret values and are descriptive rather
than inferential.

\begin{table}[H]
\caption{Final regret with \(\alpha=4\) fixed and \(T\) varied.}
\label{tab:horizon-scaling}
\centering
\footnotesize
\begin{tabular}{rrrrr}
\toprule
\(T\) & Ours & SZ-Context & OBPC-Graph & OBPC-Exp3\\
\midrule
2000  & 202.87 (6.36) & 300.10 (5.12)  & 531.41 (7.73)   & 546.86 (8.89)\\
4000  & 299.76 (8.89) & 423.60 (9.26)  & 1031.10 (5.36) & 1070.60 (10.58)\\
6000  & 378.38 (6.65) & 557.18 (11.15) & 1504.73 (9.58) & 1583.97 (8.45)\\
8000  & 429.78 (8.88) & 668.47 (11.09) & 1959.87 (11.00)& 2091.23 (8.40)\\
10000 & 476.33 (8.90) & 732.77 (12.73) & 2403.75 (12.99)& 2582.93 (7.35)\\
\bottomrule
\end{tabular}
\end{table}

\begin{table}[H]
\caption{Final regret with \(T=10^4\) fixed and \(\alpha\) varied.}
\label{tab:alpha-scaling}
\centering
\footnotesize
\begin{tabular}{rrrrr}
\toprule
\(\alpha\) & Ours & SZ-Context & OBPC-Graph & OBPC-Exp3\\
\midrule
1 & 267.61 (4.63)  & 723.53 (9.65) & 1796.83 (14.25) & 2590.00 (9.91)\\
2 & 356.71 (6.61)  & 723.53 (9.65) & 2181.09 (17.08) & 2590.00 (9.91)\\
3 & 419.40 (6.38)  & 723.53 (9.65) & 2307.20 (17.21) & 2590.00 (9.91)\\
4 & 473.49 (5.88)  & 723.53 (9.65) & 2419.22 (11.05) & 2590.00 (9.91)\\
8 & 694.40 (10.27) & 723.53 (9.65) & 2570.83 (10.66) & 2590.00 (9.91)\\
\bottomrule
\end{tabular}
\end{table}

The two sweeps use different seed families: $200$--$209$ for the horizon sweep
and $100$--$109$ for the graph sweep.  This accounts for the small difference
between their common point $T=10^4,\alpha=4$.

The per-context baselines use
$\eta_{\rm Exp3}=\sqrt{2\log K/(KT/M)}$ and
$\eta_{\rm Graph}=\sqrt{2\log K/(\alpha T/M)}$.  SZ-Context and our method use
the fixed asymptotic rules stated in \Cref{sec:experiments}; no parameter is
tuned separately at an individual sweep point.  The Schneider--Zimmert method
and Exp3 both assume standard bandit self-observation.  Accordingly,
SZ-Context and OBPC-Exp3 observe the loss of every played arm.  Relative to the
original strongly observable graph, this grants both baselines the self-loop
that arm 0 otherwise lacks.  These additional observations make the two
standard-bandit baselines applicable and favor them in the comparison; our
method and OBPC-Graph retain the original no-self-loop graph.

\end{document}